\def\year{2018}\relax

\documentclass[letterpaper]{article}

\usepackage{aaai18}
\usepackage{times}
\usepackage{helvet}
\usepackage{courier}
\usepackage{url}
\usepackage{graphicx}
\usepackage{bm}
\usepackage{amsmath}
\usepackage{amsfonts}
\usepackage{amssymb}
\usepackage{amsthm}
\usepackage{algorithm}
\usepackage{algorithmic}
\usepackage{multirow}
\usepackage{caption}
\usepackage{subcaption}
\usepackage[outdir=./]{epstopdf}
\newtheorem{theorem}{Theorem}

\begin{document}
 
\title{GraphGAN: Graph Representation Learning with Generative Adversarial Nets}

\author{Hongwei Wang\textsuperscript{1,2}, Jia Wang\textsuperscript{3}, Jialin Wang\textsuperscript{4,3}, Miao Zhao\textsuperscript{3},\\
\bf \Large  Weinan Zhang\textsuperscript{1}, Fuzheng Zhang\textsuperscript{2}, Xing Xie\textsuperscript{2} and Minyi Guo\textsuperscript{1}\thanks{M. Guo is the corresponding author.}\\
\textsuperscript{1}Shanghai Jiao Tong University, wanghongwei55@gmail.com, \{wnzhang, myguo\}@sjtu.edu.cn\\
\textsuperscript{2}Microsoft Research Asia, \{fuzzhang, xing.xie\}@microsoft.com\\
\textsuperscript{3}The Hong Kong Polytechnic University, \{csjiawang, csmiaozhao\}@comp.polyu.edu.hk\\
\textsuperscript{4}Huazhong University of Science and Technology, wangjialin@hust.edu.cn
}

\maketitle

\begin{abstract}
	The goal of graph representation learning is to embed each vertex in a graph into a low-dimensional vector space.
	Existing graph representation learning methods can be classified into two categories: \textit{generative} models that learn the underlying connectivity distribution in the graph, and \textit{discriminative} models that predict the probability of edge existence between a pair of vertices.
	In this paper, we propose \textit{GraphGAN}, an innovative graph representation learning framework unifying above two classes of methods, in which the generative model and discriminative model play a game-theoretical \textit{minimax} game.
	Specifically, for a given vertex, the generative model tries to fit its underlying true connectivity distribution over all other vertices and produces ``fake'' samples to fool the discriminative model, while the discriminative model tries to detect whether the sampled vertex is from ground truth or generated by the generative model.
	With the competition between these two models, both of them can alternately and iteratively boost their performance.
	Moreover, when considering the implementation of generative model, we propose a novel \textit{graph softmax} to overcome the limitations of traditional softmax function, which can be proven satisfying desirable properties of \textit{normalization}, \textit{graph structure awareness}, and \textit{computational efficiency}.
	Through extensive experiments on real-world datasets, we demonstrate that GraphGAN achieves substantial gains in a variety of applications, including link prediction, node classification, and recommendation, over state-of-the-art baselines.
\end{abstract}

\section{Introduction}
	\textit{Graph representation learning}, also known as \textit{network embedding}, aims to represent each vertex in a graph (network) as a low-dimensional vector, which could facilitate tasks of network analysis and prediction over vertices and edges.
	Learned embeddings are capable to benefit a wide range of real-world applications such as link prediction \cite{gao2011temporal}, node classification \cite{tang2016node}, recommendation \cite{yu2014personalized}, visualization \cite{maaten2008visualizing}, knowledge graph representation \cite{lin2015learning}, clustering \cite{tian2014learning}, text embedding \cite{tang2015pte}, and social network analysis \cite{liu2016aligning}.
	Recently, researchers have examined applying representation learning methods to various types of graphs, such as weighted graphs \cite{grover2016node2vec}, directed graphs \cite{zhou2017scalable}, signed graphs \cite{wang2017signed}, heterogeneous graphs \cite{wang2018shine}, and attributed graphs \cite{huang2017label}.
	In addition, several prior works also try to preserve specific properties during the learning process, such as global structures \cite{wang2016structural}, community structures \cite{wang2017community}, group information \cite{chen2016incorporate}, and asymmetric transitivity \cite{ou2016asymmetric}.
	
	Arguably, most existing methods of graph representation learning can be classified into two categories.
	The first is \textit{generative} graph representation learning model \cite{perozzi2014deepwalk,grover2016node2vec,zhou2017scalable,dong2017metapath2vec,li2017semi}.
	Similar to classic generative models such as Gaussian Mixture Model \cite{lindsay1995mixture} or Latent Dirichlet Allocation \cite{blei2003latent}, generative graph representation learning models assume that, for each vertex $v_c$, there exists an underlying true connectivity distribution $p_{\rm{true}} (v | v_c)$, which implies $v_c$'s connectivity preference (or relevance distribution) over all other vertices in the graph.
	The edges in the graph can thus be viewed as observed samples generated by these conditional distributions, and these generative models learn vertex embeddings by maximizing the likelihood of edges in the graph.
	For example, DeepWalk \cite{perozzi2014deepwalk} uses random walk to sample ``context'' vertices for each vertex, and tries to maximize the log-likelihood of observing context vertices for the given vertex.
	Node2vec \cite{grover2016node2vec} further extends the idea by proposing a biased random walk procedure, which provides more flexibility when generating the context for a given vertex.
	
	The second kind of graph representation learning method is the \textit{discriminative} model \cite{wang2018shine,cao2016deep,wang2016structural,li2017ppne}.
	Different from generative models, discriminative graph representation learning models do not treat edges as generated from an underlying conditional distribution, but aim to learn a classifier for predicting the existence of edges directly.
	Typically, discriminative models consider two vertices $v_i$ and $v_j$ jointly as features, and predict the probability of an edge existing between the two vertices, i.e., $p \big( edge | (v_i, v_j) \big)$, based on the training data in the graph.
	For instance, SDNE \cite{wang2016structural} uses the sparse adjacency vector of vertices as raw features for each vertex, and applies an autoencoder to extract short and condense features for vertices under the supervision of edge existence.
	PPNE \cite{li2017ppne} directly learns vertex embeddings with supervised learning on positive samples (connected vertex pairs) and negative samples (disconnected vertex pairs), also preserving the inherent properties of vertices during the learning process.
	
	Although generative and discriminative models are generally two disjoint classes of graph representation learning methods, they can be considered two sides of the same coin \cite{wang2017irgan}.
	In fact, LINE \cite{tang2015line} has done a preliminary trial on implicitly combining these two objectives (the first-order and second-order proximity, as called in LINE).
	Recently, \textit{Generative Adversarial Nets} (GAN) \cite{goodfellow2014generative} have received a great deal of attention.
	By designing a game-theoretical \textit{minimax} game to combine generative and discriminative models, GAN and its variants achieve success in various applications, such as image generation \cite{denton2015deep}, sequence generation \cite{yu2017seqgan}, dialogue generation \cite{li2017adversarial}, information retrieval \cite{wang2017irgan}, and domain adaption \cite{zhang2017aspect}.
	
	Inspired by GAN, in this paper we propose \textit{GraphGAN}, a novel framework that unifies generative and discriminative thinking for graph representation learning.
	Specifically, we aim to train two models during the learning process of GraphGAN:
	1) Generator $G(v | v_c)$, which tries to fit the underlying true connectivity distribution $p_{\rm{true}} (v | v_c)$ as much as possible, and generates the most likely vertices to be connected with $v_c$;
	2) Discriminator $D(v, v_c)$, which tries to distinguish well-connected vertex pairs from ill-connected ones, and calculates the probability of whether an edge exists between $v$ and $v_c$.
	In the proposed GraphGAN, the generator $G$ and the discriminator $D$ act as two players in a \textit{minimax} game:
	the generator tries to produce the most indistinguishable ``fake'' vertices under guidance provided by the discriminator, while the discriminator tries to draw a clear line between the ground truth and ``counterfeits'' to avoid being fooled by the generator.
	Competition in this game drives both of them to improve their capability, until the generator is indistinguishable from the true connectivity distribution.
	
	Under the GraphGAN framework, we study the choices of generator and discriminator.
	Unfortunately, we find that the traditional softmax function (and its variants) is not suitable for the generator for two reasons:
	1) softmax treats all other vertices in the graph equally for a given vertex, lacking the consideration on graph structure and proximity information;
	2) the calculation of softmax involves all vertices in the graph, which is time-consuming and computationally inefficient.
	To overcome these limitations, in GraphGAN we propose a new implementation of generator called \textit{Graph Softmax}.
	Graph softmax provides a new definition of connectivity distribution in a graph.
	We prove graph softmax satisfying desirable properties of \textit{normalization}, \textit{graph structure awareness}, and \textit{computational efficiency}.
	Accordingly, we propose a random-walk-based online generating strategy for generator, which is consistent with the definition of graph softmax and can greatly reduce computation complexity.
	
	Empirically, we apply GraphGAN to three real-world scenarios, i.e., link prediction, node classification, and recommendation, using five real-world graph-structured datasets.
	The experiment results show that GraphGAN achieves substantial gains compared with state-of-the-art baselines in the field of graph representation learning.
	Specifically, GraphGAN outperforms baselines by $0.59\%$ to $11.13\%$ in link prediction and by $0.95\%$ to $21.71\%$ in node classification both on Accuracy.
	Additionally, GraphGAN improves Precision@20 by at least $38.56\%$ and Recall@20 by at least $52.33\%$ in recommendation.
	We attribute the superiority of GraphGAN to its \textit{unified adversarial learning framework} as well as the design of the \textit{proximity-aware graph softmax} that naturally captures structural information from graphs.

\section{Graph Generative Adversarial Nets}
	In this section, we introduce the framework of GraphGAN and discuss the details of implementation and optimization of the generator and the discriminator.
	We then present the graph softmax implemented as the generator, and prove its superior properties over the traditional softmax function.

	\subsection{GraphGAN Framework}
		We formulate the generative adversarial nets for graph representation learning as follows.
		Let $\mathcal G = (\mathcal V, \mathcal E)$ be a given graph, where $\mathcal V = \{ v_1, ..., v_V \}$ represents the set of vertices and $\mathcal E = \{ e_{ij} \}_{i, j = 1}^{V}$ represents the set of edges.
		For a given vertex $v_c$, we define $\mathcal N(v_c)$ as the set of vertices directly connected to $v_c$, the size of which is typically much smaller than the total number of vertices $V$.
		We denote the underlying true connectivity distribution for vertex $v_c$ as conditional probability $p_{\rm{true}} (v | v_c)$, which reflects $v_c$'s connectivity preference distribution over all other vertices in $\mathcal V$.
		From this point of view, $\mathcal N(v_c)$ can be seen as a set of observed samples drawn from $p_{\rm{true}} (v | v_c)$.
		Given the graph $\mathcal G$, we aim to learn the following two models:
	
		\textbf{Generator $G(v | v_c; \theta_G)$}, which tries to approximate the underlying true connectivity distribution $p_{\rm{true}} (v | v_c)$, and generates (or selects, if more precise) the most likely vertices to be connected with $v_c$ from vertex set $\mathcal V$.
	
		\textbf{Discriminator $D(v, v_c; \theta_D)$}, which aims to discriminate the connectivity for the vertex pair $(v, v_c)$.
		$D(v, v_c; \theta_D)$ outputs a single scalar representing the probability of an edge existing between $v$ and $v_c$.
	
		Generator $G$ and discriminator $D$ act as two opponents: generator $G$ would try to fit $p_{\rm{true}} (v | v_c)$ perfectly and generate relevant vertices similar to $v_c$'s real immediate neighbors to deceive the discriminator, while discriminator $D$, on the contrary, would try to detect whether these vertices are ground-truth neighbors of $v_c$ or the ones generated by its counterpart $G$.
		Formally, $G$ and $D$ are playing the following two-player \textit{minimax} game with value function $V(G, D)$:
		\begin{equation}
		\label{eq:minimax}
			\begin{split}
				\min_{\theta_G} &\max_{\theta_D} V(G, D) = \sum_{c=1}^V \Big(\mathbb E_{v \sim p_{\rm{true}} (\cdot | v_c)} \big[ \log D(v, v_c; \theta_D) \big]\\
				&+ \mathbb E_{v \sim G(\cdot | v_c; \theta_G)} \big[ \log \big(1 - D(v, v_c; \theta_D) \big) \big] \Big).
			\end{split}
		\end{equation}
	
		Based on Eq. (\ref{eq:minimax}), the optimal parameters of the generator and the discriminator can be learned by alternately maximizing and minimizing the value function $V(G, D)$.
		The GraphGAN framework is illustrated as shown in Figure \ref{fig:framework}.
		In each iteration, discriminator $D$ is trained with positive samples from $p_{\rm{true}} (\cdot | v_c)$ (vertices in green) and negative samples from generator $G(\cdot | v_c; \theta_G)$ (vertices with blue stripes), and generator $G$ is updated with policy gradient under the guidance of $D$ (detailed later in this section).
		Competition between $G$ and $D$ drives both of them to improve their methods until $G$ is indistinguishable from the true connectivity distribution.
		We discuss the implementation and optimization of $D$ and $G$ as follows.
	
		\begin{figure}
			\centering
			\includegraphics[width=0.45\textwidth]{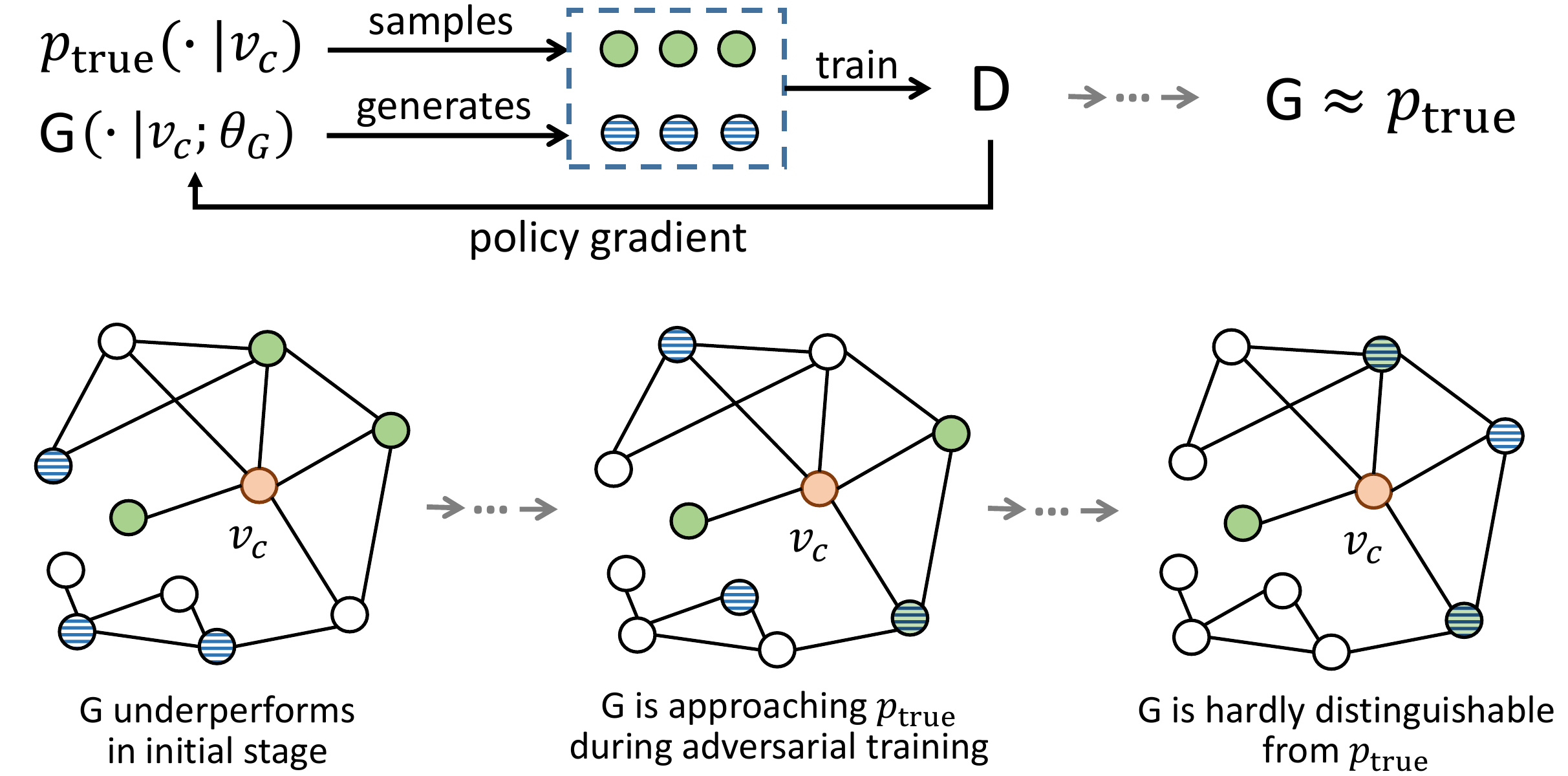}
			\caption{Illustration of GraphGAN framework.}
			\label{fig:framework}
		\end{figure}

	\subsection{Discriminator Optimization}
		Given positive samples from true connectivity distribution and negative samples from the generator, the objective for the discriminator is to maximize the log-probability of assigning the correct labels to both positive and negative samples, which could be solved by stochastic gradient ascent if $D$ is differentiable with respect to $\theta_D$.
		In GraphGAN, we define $D$ as the sigmoid function of the inner product of two input vertices:
		\begin{equation}
			\label{eq:d}
			D(v, v_c) = \sigma ({\bf d}_v^\top {\bf d}_{v_c}) = \frac{1}{1 + \exp(-{\bf d}_v^\top {\bf d}_{v_c})},
		\end{equation}
		where ${\bf d}_v, {\bf d}_{v_c} \in \mathbb R^k$ are the \textit{k}-dimensional representation vectors of vertices $v$ and $v_c$ respectively for discriminator $D$, and $\theta_D$ is the union of all ${\bf d}_v$'s.
		Any discriminative model can serve as $D$ here such as SDNE \cite{wang2016structural}, and we leave the further study of choice of discriminator for future work.
		Note that Eq. (\ref{eq:d}) simply involves $v$ and $v_c$, which indicates that given a sample pair $(v, v_c)$, we need to update only ${\bf d}_v$ and ${\bf d}_{v_c}$ by ascending the gradient with respect to them:
		\begin{equation}
			\label{eq:update_d}
			\nabla_{\theta_D} V(G, D) =
			\begin{cases}
				\nabla_{\theta_D} \log D(v, v_c), \ if \ v \sim p_{\rm{true}};\\
				\nabla_{\theta_D} \big( 1 - \log D(v, v_c) \big), \ if \ v \sim G.
			\end{cases}
		\end{equation}

	\subsection{Generator Optimization}
		In contrast to discriminator, the generator aims to minimize the log-probability that the discriminator correctly assigns negative labels to the samples generated by $G$.
		In other words, the generator shifts its approximated connectivity distribution (through its parameters $\theta_G$) to increase the scores of its generated samples, as judged by $D$.
		Because the sampling of $v$ is discrete, following \cite{schulman2015gradient,yu2017seqgan}, we propose computing the gradient of $V(G, D)$ with respect to $\theta_G$ by policy gradient:
		\begin{equation}
			\label{eq:policy_gradient}
			\begin{split}
				&\nabla_{\theta_G} V(G, D)\\
				=& \nabla_{\theta_G} \sum_{c=1}^V \mathbb E_{v \sim G(\cdot | v_c)} \big[ \log \big(1 - D(v, v_c) \big) \big]\\
				=& \sum_{c=1}^V \sum_{i=1}^N \nabla_{\theta_G} G(v_i | v_c) \log \big(1 - D(v_i, v_c) \big)\\
				=& \sum_{c=1}^V \sum_{i=1}^N G(v_i | v_c) \nabla_{\theta_G} \log G(v_i | v_c) \log \big(1 - D(v_i, v_c) \big)\\
				=& \sum_{c=1}^V \mathbb E_{v \sim G(\cdot | v_c)} \big[ \nabla_{\theta_G} \log G(v | v_c) \log \big(1 - D(v, v_c) \big) \big].
			\end{split}
		\end{equation}
		To understand the above formula, it is worth noting that gradient $\nabla_{\theta_G} V(G, D)$ is an expected summation over the gradients $\nabla_{\theta_G} \log G(v | v_c; \theta_G)$ weighted by log-probability $\log \big( 1 - D(v, v_c; \theta_D) \big)$, which, intuitively speaking, indicates that vertices with a higher probability of being negative samples will ``tug'' generator $G$ stronger away from themselves, since we apply gradient descent on $\theta_G$.
	
		We now discuss the implementation of $G$.
		A straightforward way is to define the generator as a softmax function over all other vertices \cite{wang2017irgan}, i.e.,
		\begin{equation}
			\label{eq:softmax}
			G(v | v_c) = \frac{\exp ({\bf g}_v^\top {\bf g}_{v_c})}{\sum_{v \neq v_c} \exp ({\bf g}_v^\top {\bf g}_{v_c})},
		\end{equation}
		where ${\bf g}_v, {\bf g}_{v_c} \in \mathbb R^k$ are the \textit{k}-dimensional representation vectors of vertex $v$ and $v_c$ respectively for generator $G$, and $\theta_G$ is the union of all ${\bf g}_v$'s.
		Under this setting, to update $\theta_G$ in each iteration, we calculate the approximated connectivity distribution $G(v | v_c; \theta_G)$ based on Eq. (\ref{eq:softmax}), draw a set of samples $(v, v_c)$ randomly according to $G$, and update $\theta_G$ by stochastic gradient descent.
		Softmax provides a concise and intuitive definition for the connectivity distribution in $G$, but it has two limitations in graph representation learning:
		1) The calculation of softmax in Eq. (\ref{eq:softmax}) involves all vertices in the graph, which implies that for each generated sample $v$, we need to calculate gradients $\nabla_{\theta_G} \log G(v | v_c; \theta_G)$ and update all vertices.
		This is computationally inefficient, especially for real-world large-scale graphs with millions of vertices.
		2) The graph structure encodes rich information of proximity among vertices, but softmax completely ignores the utilization of structural information from graphs as it treats vertices without any discrimination.
		Recently, hierarchical softmax \cite{morin2005hierarchical} and negative sampling \cite{mikolov2013distributed} are popular alternatives to softmax.
		Although these methods can alleviate the computation to some extent, neither of them considers structural information of a graph, thereby being unable to achieve satisfactory performance when applied to graph representation learning.

	\subsection{Graph Softmax for Generator}
		To address the aforementioned problems, in GraphGAN we propose a new alternative to softmax for the generator called \textit{graph softmax}.
		The key idea of graph softmax is to define a new method of computing connectivity distribution in generator $G(\cdot | v_c; \theta_G)$ that satisfies the following three desirable properties:
		\begin{itemize}
			\item
				\textit{Normalized}.
				The generator should produce a valid probability distribution, i.e., $\sum_{v \neq v_c} G(v | v_c; \theta_G) = 1$.
			\item
				\textit{Graph-structure-aware}.
				The generator should take advantage of the structural information of a graph to approximate the true connectivity distribution.
				Intuitively, for two vertices in a graph, their connectivity probability should decline with the increase of their shortest distance.
			\item
				\textit{Computationally efficient}.
				Distinguishable from full softmax, the computation of $G(v | v_c; \theta_G)$ should only involve a small number of vertices in the graph.
		\end{itemize}
		
		We discuss graph softmax in detail as follows.
		To calculate the connectivity distribution $G(\cdot | v_c; \theta_G)$, we first perform Breadth First Search (BFS) on the original graph $\mathcal G$ starting from vertex $v_c$, which provides us with a BFS-tree $T_c$ rooted at $v_c$.
		Given $T_c$, we denote $\mathcal N_c(v)$ as the set of neighbors of $v$ (i.e., vertices that are directly connected to $v$) in $T_c$, including its parent vertex and all child vertices if exist.
		For a given vertex $v$ and one of its neighbors $v_i \in \mathcal N_c(v)$, we define the relevance probability of $v_i$ given $v$ as
		\begin{equation}
			\label{eq:p_c}
			p_c (v_i | v) = \frac{\exp ({\bf g}_{v_i}^\top {\bf g}_v)}{\sum_{v_j \in \mathcal N_c(v)} \exp ({\bf g}_{v_j}^\top {\bf g}_v)},
		\end{equation}
		which is actually a softmax function over $\mathcal N_c(v)$.
		To calculate $G(v | v_c; \theta_G)$, note that each vertex $v$ can be reached by a unique path from the root $v_c$ in $T_c$.
		Denote the path as $P_{v_c \rightarrow v} = (v_{r_0}, v_{r_1}, ..., v_{r_m})$ where $v_{r_0} = v_c$ and $v_{r_m} = v$.
		Then the graph softmax defines $G(v | v_c; \theta_G)$ as follows:
		\begin{equation}
			\label{eq:generator}
			G(v | v_c) \triangleq \big( \prod\nolimits_{j=1}^m p_c(v_{r_j} | v_{r_{j-1}}) \big) \cdot p_c(v_{r_{m-1}} | v_{r_m}),
		\end{equation}
		where $p_c (\cdot | \cdot)$ is the relevance probability defined in Eq. (\ref{eq:p_c}).
		
		We prove that our proposed graph softmax satisfies the above three properties, i.e., graph softmax is \textit{normalized}, \textit{graph-structure-aware}, and \textit{computationally efficient}.
		
		\begin{theorem}
			$\sum_{v \neq v_c} G(v | v_c; \theta_G) = 1$ in graph softmax.
		\end{theorem}
		
		\begin{proof}
			Before proving the theorem, we first give a proposition as follows.
			Denote $ST_v$ as the sub-tree rooted at $v$ in $T_c$ ($v \neq v_c$).
			Then we have
			\begin{equation}
				\label{eq:proposition}
				\sum\nolimits_{v_i \in ST_v} G(v_i | v_c) = \prod\nolimits_{j=1}^m p_c(v_{r_j} | v_{r_{j-1}}),
			\end{equation}
			where $(v_{r_0}, v_{r_1}, ..., v_{r_m})$ is on the path $P_{v_c \rightarrow v}$, $v_{r_0} = v_c$ and $v_{r_m} = v$.			
			The proposition can be proved by bottom-up induction on the BFS-tree $T_c$:
			\begin{itemize}
				\item
					For each leaf vertex $v$, we have $\sum_{v_i \in ST_v} G(v_i | v_c; \theta_G) = G(v | v_c; \theta_G) = \big( \prod_{j=1}^m p_c(v_{r_j} | v_{r_{j-1}}) \big) \cdot p_c(v_{r_{m-1}} | v_{r_m}) = \prod_{j=1}^m p_c(v_{r_j} | v_{r_{j-1}})$.
					The last step is due to the fact that leaf vertex $v$ has only one neighbor (parent vertex $v_{r_{m-1}}$), therefore, $p_c(v_{r_{m-1}} | v_{r_m}) = p_c(v_{r_{m-1}} | v) = 1$.
				\item
					For each non-leaf vertex $v$, we denote $\mathcal C_c(v)$ as the set of children vertices of $v$ in $T_c$.
					By induction hypothesis, each children vertex $v_k \in \mathcal C_c(v)$ satisfies the proposition in Eq. (\ref{eq:proposition}).
					Thus we have
					\begin{equation}
						\nonumber
						\begin{split}
							&\sum\nolimits_{v_i \in ST_v} G(v_i | v_c)\\
							=& G(v | v_c) + \sum\nolimits_{v_k \in \mathcal C_c(v)} \sum\nolimits_{v_i \in ST_{v_k}} G(v_i | v_c)\\
							=& \big( \prod\nolimits_{j=1}^m p_c(v_{r_j} | v_{r_{j-1}}) \big) p_c(v_{r_{m-1}} | v_{r_m})\\
							&+ \sum\nolimits_{v_k \in \mathcal C_c(v)} \Big( \big( \prod\nolimits_{j=1}^m p_c(v_{r_j} | v_{r_{j-1}}) \big) p_c(v_k | v_{r_m}) \Big)\\
							=& \big( \prod_{j=1}^m p_c(v_{r_j} | v_{r_{j-1}}) \big) \big( p_c(v_{r_{m-1}} | v) + \sum_{v_k \in \mathcal C_c(v)} p_c(v_k | v) \big)\\
							=& \prod\nolimits_{j=1}^m p_c(v_{r_j} | v_{r_{j-1}}).
						\end{split}
					\end{equation}
			\end{itemize}
			So far we have proven Eq. (\ref{eq:proposition}).
			Applying Eq. (\ref{eq:proposition}) to all children vertices of $v_c$, we have
			$\sum_{v \neq v_c} G(v | v_c; \theta_G) = \sum_{v_k \in \mathcal C_c(v_c)} \sum_{v \in ST_{v_k}} G(v | v_c; \theta_G) = \sum_{v_k \in \mathcal C_c(v_c)} p_c(v_k | v_c)\\= 1$.
		\end{proof}
		
		\begin{theorem}
		\label{thm:2}
			In graph softmax, $G(v | v_c; \theta_G)$ decreases exponentially with the increase of the shortest distance between $v$ and $v_c$ in original graph $\mathcal G$.
		\end{theorem}
		
		\begin{proof}
			According to the definition of graph softmax, $G(v | v_c; \theta_G)$ is the product of $m+1$ terms of relevance probability, where $m$ is the length of path $P_{v_c \rightarrow v}$.
			Note that $m$ is also the shortest distance between $v_c$ and $v$ in graph $\mathcal G$, since BFS-tree $T_c$ preserves the shortest distances between $v_c$ and all other vertices in the original graph.
			Therefore, we conclude that $G(v | v_c; \theta_G)$ is exponentially proportional to the inverse of the shortest distance between $v$ and $v_c$ in $\mathcal G$.
		\end{proof}
		Following Theorem \ref{thm:2}, we further justify that graph softmax characterizes the real pattern of connectivity distribution precisely by conducting an empirical study in the experiment part.
		
		\begin{theorem}
			In graph softmax, calculation of $G(v | v_c; \theta_G)$ depends on $O(d \log V)$ vertices, where $d$ is average degree of vertices and $V$ is the number of vertices in graph $\mathcal G$.
		\end{theorem}
		
		\begin{proof}
			According to Eq. (\ref{eq:p_c}) and Eq. (\ref{eq:generator}), the calculation of $G(v | v_c; \theta_G)$ involves two types of vertices: vertices on the path $P_{v_c \rightarrow v}$ and vertices directly connected to the path (i.e., vertices whose distance from the path is 1).
			In general, the maximal length of the path is $\log V$, which is the depth of the BFS-tree, and each vertex in the path is connected to $d$ vertices on average.
			Therefore, the total number of involved vertices in $G(v | v_c; \theta_G)$ is $O(d \log V)$.
		\end{proof}
		
		\begin{figure*}
			\centering
			\includegraphics[width=1.0\textwidth]{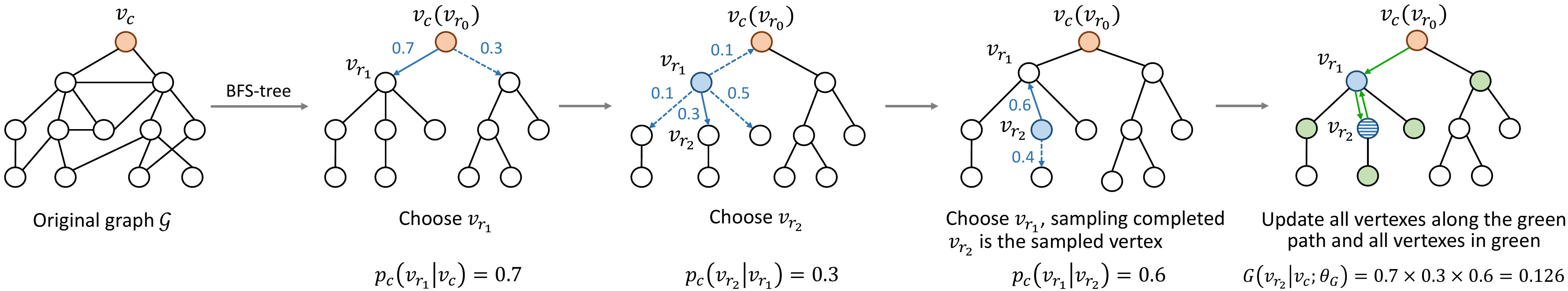}
			\caption{Online generating strategy for generator $G$. The blue digits are the relevance probability $p_c(v_i | v)$, and the blue solid arrow indicates the direction that $G$ chooses to move in. Upon completion of sampling, the vertex with blue stripes is the sampled one, and all colored vertices in the rightmost tree require updating accordingly.}
			\label{fig:graph_softmax}
		\end{figure*}
		
		Next, we discuss the generating (or sampling) strategy for generator $G$.
		A feasible way of generating vertices is to calculate $G(v | v_c; \theta_G)$ for all vertices $v \neq v_c$, and perform random sampling proportionally to their approximated connectivity probabilities.
		Here we propose an online generating method, which is more computationally efficient and consistent with the definition of graph softmax.
		To generate a vertex, we perform a random walk starting at the root $v_c$ in $T_c$ with respect to the transition probability defined in Eq. (\ref{eq:p_c}).
		During the process of random walk, if the currently visited vertex is $v$ and generator $G$ decides to visit $v$'s parent (i.e., turning around on the path) for the first time, then $v$ is chosen as the generated vertex.
		
		The online generating strategy for the generator is formally described in Algorithm \ref{alg:generating_strategy}.
		We denote the currently visited vertex as $v_{cur}$ and the previously visited one as $v_{pre}$.
		Note that Algorithm \ref{alg:generating_strategy} terminates in $O(\log V)$ steps, since the random-walk path will turn around at the latest when reaching leaf vertices.
		Similar to the computation of graph softmax, the complexity of the above online generating method is $O(d \log V)$, which is significantly lower than the offline method with a complexity of $O(V \cdot d \log V)$.
		Figure \ref{fig:graph_softmax} gives an illustrative example of the generating strategy as well as the computation process of graph softmax.
		In each step of a random walk, a blue vertex $v_{i}$ is chosen out of all neighbors of $v_{cur}$ by random selection proportionally to the relevance probability $p_c(v_i | v_{cur})$ defined in Eq. (\ref{eq:p_c}).
		Once $v_i$ equals $v_{pre}$, i.e., the random walk revisits $v_{cur}$'s parent $v_{pre}$, $v_{cur}$ would be sampled out (indicated as the vertex with blue stripes in the figure), and all vertices along the path $P_{v_c \rightarrow v_{cur}}$ as well as the vertices that are directly connected to this path need to be updated according to Eq. (\ref{eq:policy_gradient}), (\ref{eq:p_c}) and (\ref{eq:generator}). 

		\setlength{\textfloatsep}{10pt}
		\begin{algorithm}[t]
			\small
			\caption{Online generating strategy for the generator}
			\label{alg:generating_strategy}
			\begin{algorithmic}[1]
				\REQUIRE{BFS-tree $T_c$, representation vectors $\{{\bf g}_i\}_{i \in \mathcal V}$}
				\ENSURE{generated sample $v_{gen}$}
				\STATE $v_{pre} \leftarrow v_c$, $v_{cur} \leftarrow v_c$;
				\WHILE{\TRUE}
					\STATE Randomly select $v_i$ proportionally to $p_c(v_i | v_{cur})$ in Eq. (\ref{eq:p_c});
					\IF{$v_i = v_{pre}$}
						\STATE $v_{gen} \leftarrow v_{cur}$;
						\RETURN $v_{gen}$
					\ELSE
						\STATE	$v_{pre} \leftarrow v_{cur}$, $v_{cur} \leftarrow v_i$;
					\ENDIF
				\ENDWHILE
			\end{algorithmic}
		\end{algorithm}
		
		\begin{algorithm}[t]
			\small
			\caption{GraphGAN framework}
			\label{alg:graph_gan}
			\begin{algorithmic}[1]
				\REQUIRE{dimension of embedding $k$, size of generating samples $s$, size of discriminating samples $t$}
				\ENSURE{generator $G(v | v_c; \theta_G)$, discriminator $D(v, v_c; \theta_D)$}
				\STATE Initialize and pre-train $G(v | v_c; \theta_G)$ and $D(v, v_c; \theta_D)$;
				\STATE Construct BFS-tree $T_c$ for all $v_c \in \mathcal V$;
				\WHILE{GraphGAN not converge}
					\FOR{G-steps}
						\STATE $G(v | v_c; \theta_G)$ generates $s$ vertices for each vertex $v_c$ according to Algorithm \ref{alg:generating_strategy};
						\STATE Update $\theta_G$ according to Eq. (\ref{eq:policy_gradient}), (\ref{eq:p_c}) and (\ref{eq:generator});
					\ENDFOR
					\FOR{D-steps}
						\STATE Sample $t$ positive vertices from ground truth and $t$ negative vertices from $G(v | v_c; \theta_G)$ for each vertex $v_c$;
						\STATE Update $\theta_D$ according to Eq. (\ref{eq:d}) and (\ref{eq:update_d});
					\ENDFOR
				\ENDWHILE
				\RETURN $G(v | v_c; \theta_G)$ and $D(v, v_c; \theta_D)$
			\end{algorithmic}
		\end{algorithm}

		Finally, the overall logic of GraphGAN is summarized in Algorithm \ref{alg:graph_gan}.
		We provide time complexity analysis of GraphGAN as follows.
		The complexity of BFS-tree construction for all vertices in line 2 is $O \big( V(V+E) \big) = O(dV^2)$ since the time complexity of BFS is $O(V+E)$ \cite{cormen2009introduction}.
		In each iteration, the complexity of both line 5 and line 6 is $O(sV \cdot d\log V \cdot k)$, and the complexity of line 9 and line 10 is $O(tV \cdot d\log V \cdot k)$ and $O(tV \cdot k)$, respectively.
		In general, if we treat $k$, $s$, $t$, and $d$ as constants, the complexity of each iteration in GraphGAN is $O(V\log V)$.

\section{Experiments}
	In this section, we evaluate the performance of GraphGAN\footnote{https://github.com/hwwang55/GraphGAN} on a series of real-world datasets.
	Specifically, we choose three application scenarios for experiments, i.e., link prediction, node classification, and recommendation.

	\subsection{Experiments Setup}
		We utilize the following five datasets in our experiments:
		\begin{itemize}
		 	\setlength\itemsep{0.0em}
			\item
				arXiv-AstroPh\footnote{https://snap.stanford.edu/data/ca-AstroPh.html} is from the e-print arXiv and covers scientific collaborations between authors with papers submitted to the Astro Physics category.
				The vertices represent authors and the edge indicates co-author relationship.
				This graph has 18,772 vertices and 198,110 edges.
			\item
				arXiv-GrQc\footnote{https://snap.stanford.edu/data/ca-GrQc.html} is also from arXiv and covers scientific collaborations between authors with papers submitted to the General Relativity and Quantum Cosmology categories.
				This graph has 5,242 vertices and 14,496 edges.
			\item
				BlogCatalog\footnote{http://socialcomputing.asu.edu/datasets/BlogCatalog} is a network of social relationships of the bloggers listed on the BlogCatalog website.
				The labels of vertices represent blogger interests inferred through the metadata provided by the bloggers.
				This graph has 10,312 vertices, 333,982 edges, and 39 different labels.
			\item
				Wikipedia\footnote{http://www.mattmahoney.net/dc/textdata} is a co-occurrence network of words appearing in the first $10^9$ bytes of the Eglish Wikipedia dump.
				The labels represent the inferred Part-of-Speech (POS) tags of words.
				This graph has 4,777 vertices, 184,812 edges, and 40 different labels.
			\item
				MovieLens-1M\footnote{https://grouplens.org/datasets/movielens/1m/} is a bipartite graph consisting of approximately 1 million ratings (edges) with 6,040 users and 3,706 movies in MovieLens website.
		\end{itemize}
		
		We compare our proposed GraphGAN with the following four baselines for graph representation learning:
		\begin{itemize}
			\setlength\itemsep{0.0em}
			\item
				DeepWalk \cite{perozzi2014deepwalk} adopts random walk and Skip-Gram to learn vertex embeddings.
			\item
				LINE \cite{tang2015line} preserves the first-order and second-order proximity among vertices in the graph.
			\item
				Node2vec \cite{grover2016node2vec} is a variant of DeepWalk and designs a biased random walk to learn vertex embeddings.
			\item
				Struc2vec \cite{ribeiro2017struc2vec} captures the structural identity of vertices in a graph.
		\end{itemize}
		
		For all three experiment scenarios, we perform stochastic gradient descent to update parameters in GraphGAN with learning rate $0.001$.
		In each iteration, we set $s$ as 20 and $t$ as the number of positive samples in the test set for each vertex, then run G-steps and D-steps for 30 times, respectively.
		The dimension of representation vectors $k$ for all methods is set as $20$.
		The above hyper-parameters are chosen by cross validation.
		The final learned vertex representations are ${\bf g}_i$'s.
		Parameter settings for all baselines are as default.

	\subsection{Empirical Study}
		\begin{figure}
			\centering
			\begin{subfigure}[b]{0.23\textwidth}
   				\includegraphics[width=\textwidth]{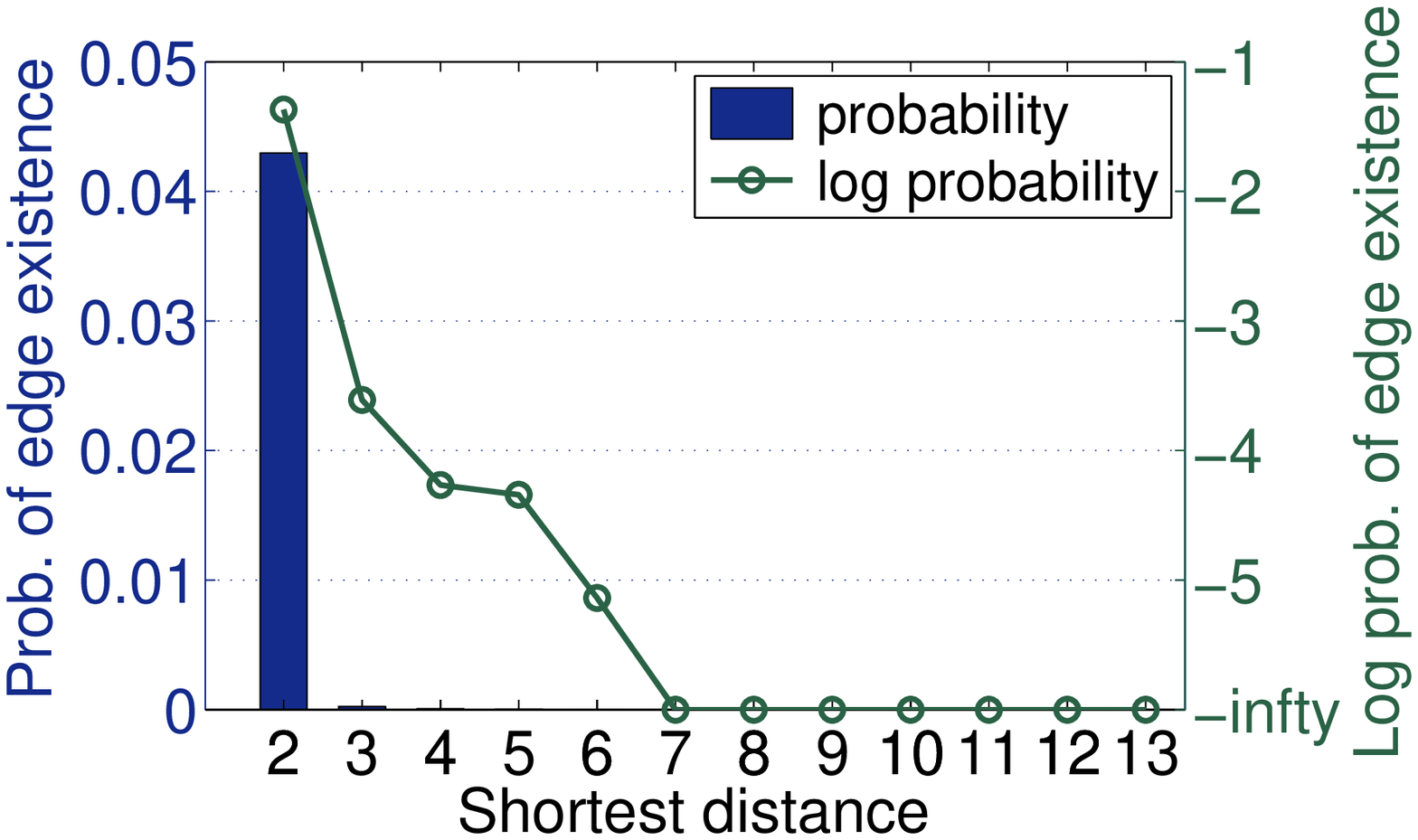}
   				\caption{arXiv-AstroPh}
   				\label{fig:empirical_study_1}
			\end{subfigure}
			\hfill
			\begin{subfigure}[b]{0.23\textwidth}
				\includegraphics[width=\textwidth]{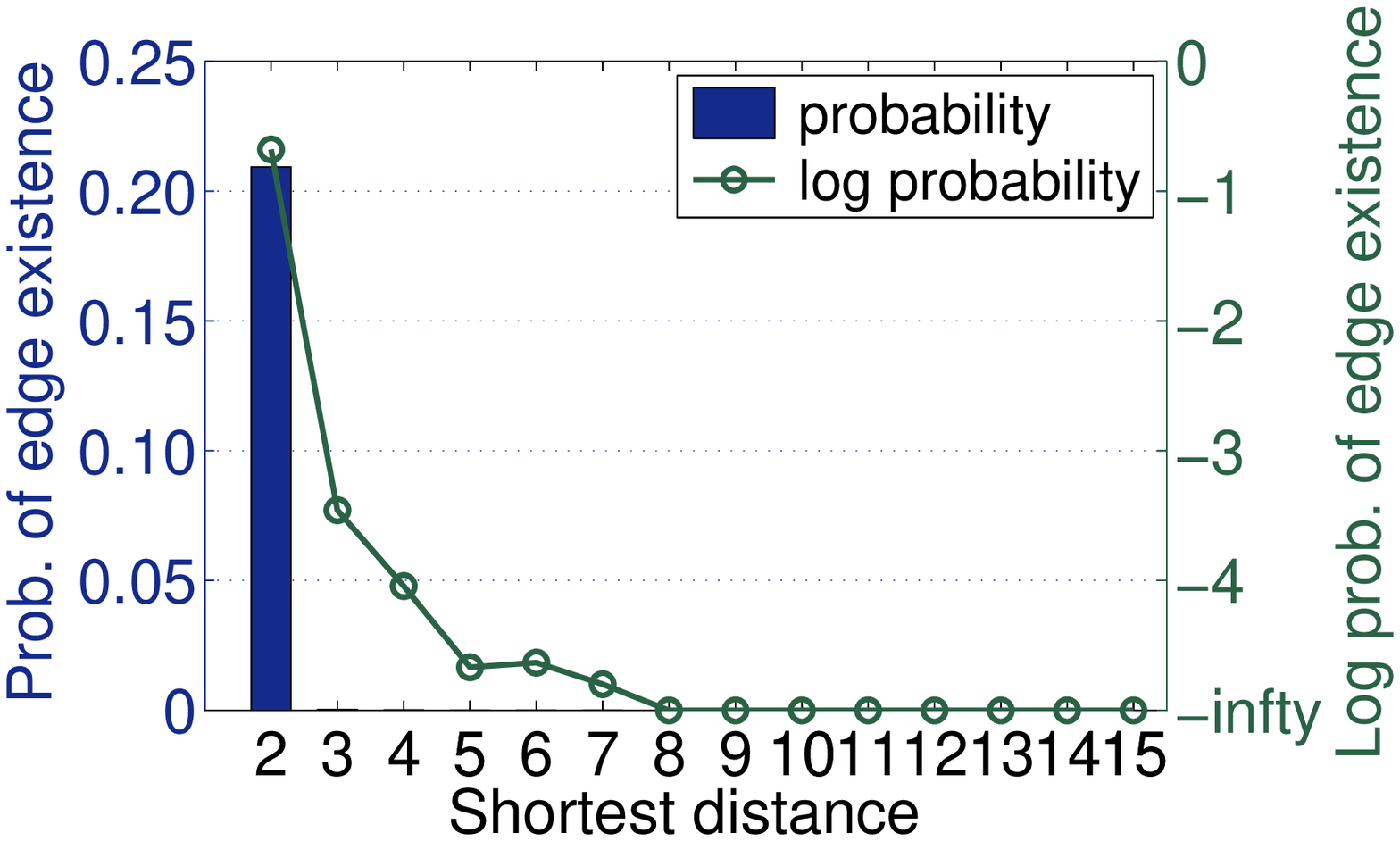}
				\caption{arXiv-GrQc}
				\label{fig:empirical_study_2}
			\end{subfigure}
			\caption{The correlation between the probability of edge existence and the shortest distance for a given vertex pair.}			
			\label{fig:empirical_study}
		\end{figure}
		
		We conduct an empirical study to investigate the real pattern of connectivity distribution in graphs.
		Specifically, for a given vertex pair, we aim to reveal how the probability of edge existence changes with their shortest distance in the graph.
		To achieve this, we first randomly sample 1 million vertex pairs from arXiv-AstroPh and arXiv-GrQc datasets, respectively.
		For each selected vertex pair, we remove the edge between them if it exists (because it is treated as hidden ground truth), and calculate their shortest distance.
		We count the probability of edge existence for all possible shortest distances, and plot the results in Figure \ref{fig:empirical_study} (the disconnected case is omitted).
		It is evident that the probability of edge existence between vertex pair drops dramatically with the increase of their shortest distance.
		We also plot the log probability curves in Figure \ref{fig:empirical_study}, which generally trends towards linear decline with $R^2=0.831$ and $0.710$.
		The above finding empirically demonstrates that the probability of edge existence between a pair of vertices is approximately exponentially proportional to the inverse of their shortest distance, which strongly proves that graph softmax captures the essence of real-world graphs according to Theorem \ref{thm:2}.

	\subsection{Link Prediction}
		\begin{table}[t]
			\setlength{\abovecaptionskip}{3pt}
			\small
                	\centering
                	\caption{Accuracy and Macro-F1 on arXiv-AstroPh and arXiv-GrQc in link prediction.}
                	\begin{tabular}{|c|c|c|c|c|}
                    	\hline
                    	\multirow{2}{*}{Model} & \multicolumn{2}{c|}{arXiv-AstroPh} & \multicolumn{2}{c|}{arXiv-GrQc} \\
                    	\cline{2-5}
                    	& Acc & Macro-F1 & Acc & Macro-F1 \\
                    	\hline
                    	DeepWalk & 0.841 & 0.839 & 0.803 & 0.812 \\
                    	\hline
                    	LINE & 0.820 & 0.814 & 0.764 & 0.761 \\
                    	\hline
                    	Node2vec & 0.845 & 0.854 & 0.844 & 0.842  \\
                    	\hline
                    	Struc2vec & 0.821 & 0.810 & 0.780 & 0.776 \\
                    	\hline
                    	GraphGAN & \textbf{0.855} & \textbf{0.859} & \textbf{0.849} & \textbf{0.853} \\
                    	\hline
			\end{tabular}
			\label{table:link_prediction}
		\end{table}
		
		In link prediction, our goal is to predict whether there exists an edge between two given vertices.
		Therefore, this task shows the performance of edge predictability of different graph representation learning methods.
		We randomly hide $10\%$ of edges in the original graph as ground truth, and use the left graph to train all graph representation learning models.
		After training, we obtain the representation vectors for all vertices and use logistic regression method to predict the probability of edge existence for a given vertex pair.
		Our test set consists of the hidden $10\%$ vertex pairs (edges) in the original graph as the positive samples and randomly selected disconnected vertex pairs as negative samples with equal number.
		We use arXiv-AstroPh and arXiv-GrQc as datasets, and report the results of \textit{Accuracy} and \textit{Macro-F1} in Table \ref{table:link_prediction}.
		We have the following observations:
		1) Performance of LINE and struc2vec is relatively poor in link prediction, as they cannot quite capture the pattern of edge existence in graphs.
		2)  DeepWalk and node2vec perform better than LINE and struc2vec.
		This is probably because DeepWalk and node2vec both utilize the random-walk-based Skip-Gram model, which is better at extracting proximity information among vertices.
		3) GraphGAN outperforms all the baselines in link prediction.
		Specifically, GraphGAN improves Accuracy on arXiv-AstroPh and arXiv-GrQc by $1.18\%$ to $4.27\%$ and $0.59\%$ to $11.13\%$, respectively.
		Our explanation is that adversarial training provides GraphGAN a higher learning flexibility than the single-model training for baselines.
		
		\begin{figure}
			\centering
			\begin{subfigure}[b]{0.23\textwidth}
   				\includegraphics[width=\textwidth]{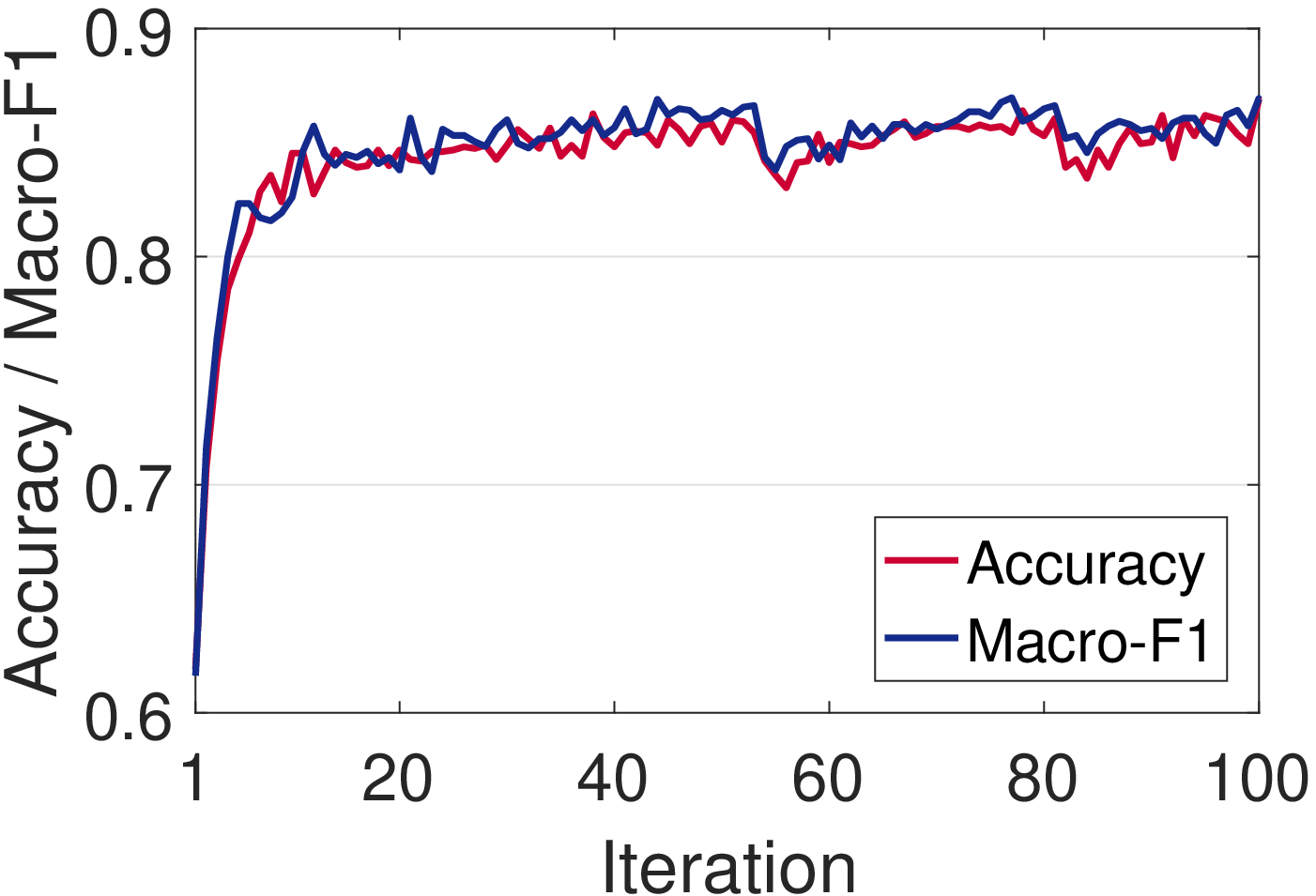}
   				\caption{Generator}
   				\label{fig:generator}
			\end{subfigure}
			\hfill
			\begin{subfigure}[b]{0.23\textwidth}
				\includegraphics[width=\textwidth]{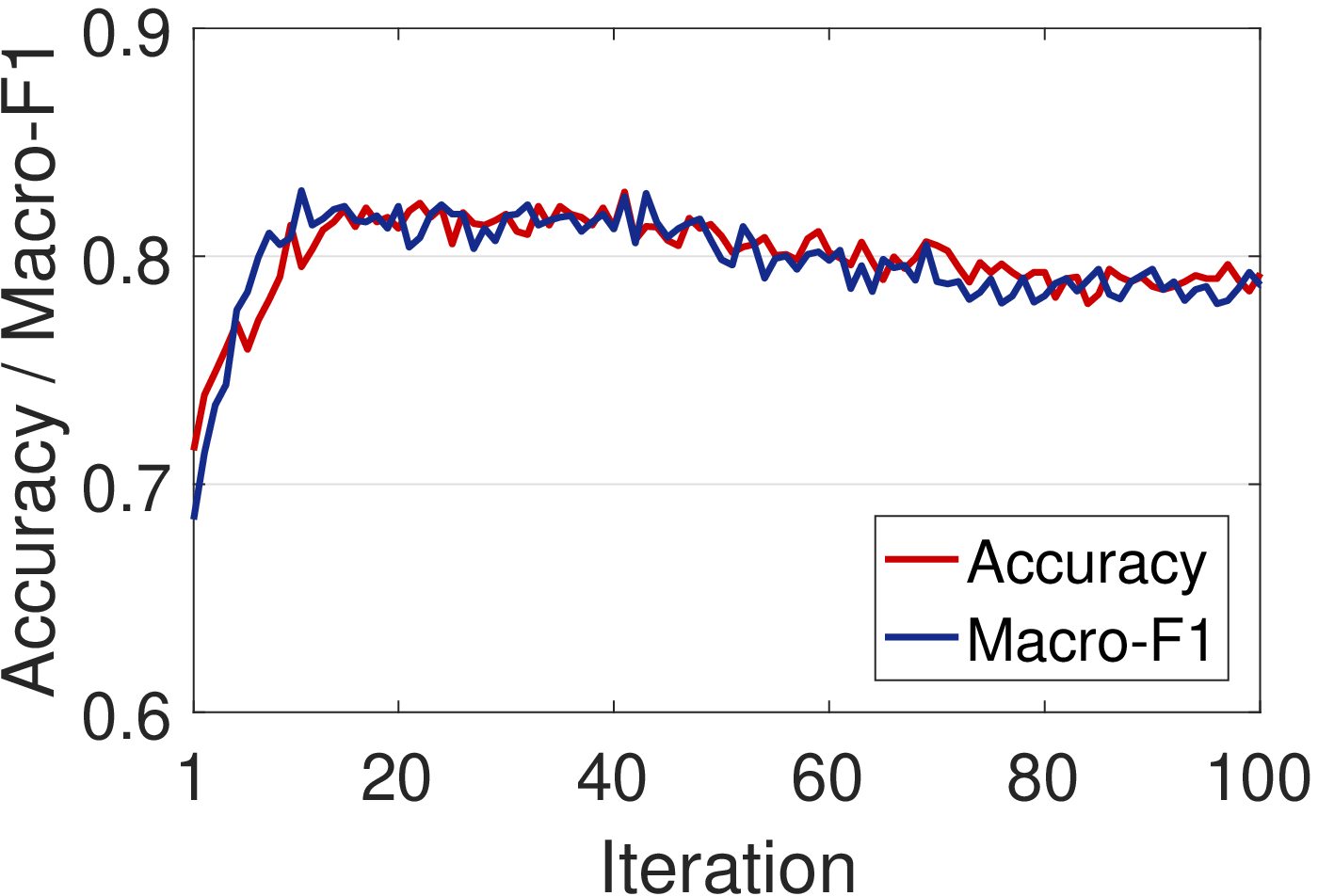}
				\caption{Discriminator}
				\label{fig:discriminator}
			\end{subfigure}
			\caption{Learning curves of the generator and the discriminator of GraphGAN on arXiv-GrQc in link prediction.}			
			\label{fig:learning_curves}
		\end{figure}
		
		To intuitively understand the learning stability of GraphGAN, we further illustrate the learning curves of the generator and the discriminator on arXiv-GrQc in Figure \ref{fig:learning_curves}.
		From Figure \ref{fig:learning_curves} we observe that the minimax game in GraphGAN arrives at an equilibrium where the generator performs outstandingly well after convergence, while performance of the discriminator boosts at first but gradually falls below 0.8.
		Note that the discriminator does not degrade to a random-guess level, because the generator still provides lots of true negative samples in practice.
		The result suggests that, different with IRGAN \cite{wang2017irgan}, the design of graph softmax enables the generator in GraphGAN to draw samples and learn vertex embeddings more efficiently.

	\subsection{Node Classification}
		\begin{table}[t]
			\setlength{\abovecaptionskip}{3pt}
			\small
                	\centering
                	\caption{Accuracy and Macro-F1 on BlogCatalog and Wikipedia in node classification.}
                	\begin{tabular}{|c|c|c|c|c|}
                    	\hline
                    	\multirow{2}{*}{Model} & \multicolumn{2}{c|}{BlogCatalog} & \multicolumn{2}{c|}{Wikipedia} \\
                    	\cline{2-5}
                    	& Acc & Macro-F1 & Acc & Macro-F1 \\
                    	\hline
                    	DeepWalk & 0.225 & 0.214 & 0.194 & 0.183 \\
                    	\hline
                    	LINE & 0.205 & 0.192 & 0.175 & 0.164 \\
                    	\hline
                    	Node2vec & 0.215 & 0.206 & 0.191 & 0.179  \\
                    	\hline
                    	Struc2vec & 0.228 & 0.216 & 0.211 & 0.190 \\
                    	\hline
                    	GraphGAN & \textbf{0.232} & \textbf{0.221} & \textbf{0.213} & \textbf{0.194} \\
                    	\hline
			\end{tabular}
			\label{table:node_classification}
		\end{table}
		
		In node classification, each vertex is assigned one or multiple labels.
		After we observe a fraction of vertices and their labels, we aim to predict labels for the remaining vertices.
		Therefore, the performance of node classification can reveal the distinguishability of vertices under different graph representation learning methods.
		To conduct the experiment, we train GraphGAN and baselines on the whole graph to obtain vertex representations, and use logistic regression as classifier to perform node classification with 9:1 train-test ratio.
		We use BlogCatalog and Wikipedia as datasets.
		The results of \textit{Accuracy} and \textit{Macro-F1} are presented in Table \ref{table:node_classification}.
		As we can see, GraphGAN outperforms all baselines on both datasets.
		For example, GraphGAN achieves gains of $1.75\%$ to $13.17\%$ and $0.95\%$ to $21.71\%$ on Accuracy on two datasets, respectively.
		This indicates that though GraphGAN is directly designed to optimize the approximated connectivity distribution on edges, it can still effectively encode the information of vertices into the learned representations.

	\subsection{Recommendation}
		\begin{figure}
			\centering
			\begin{subfigure}[b]{0.23\textwidth}
   				\includegraphics[width=\textwidth]{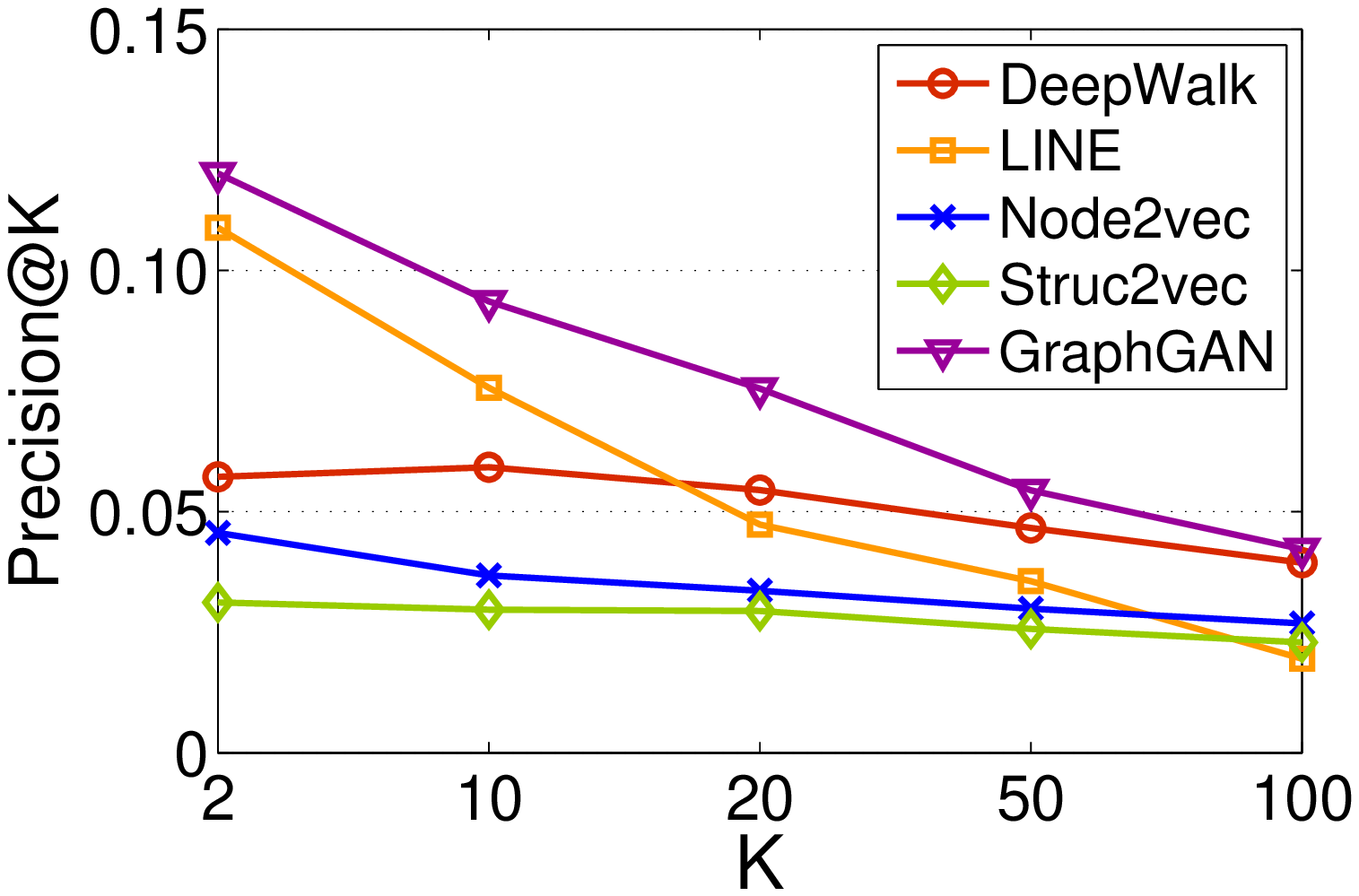}
   				\caption{Precision@K}
   				\label{fig:precision}
			\end{subfigure}
			\hfill
			\begin{subfigure}[b]{0.23\textwidth}
				\includegraphics[width=\textwidth]{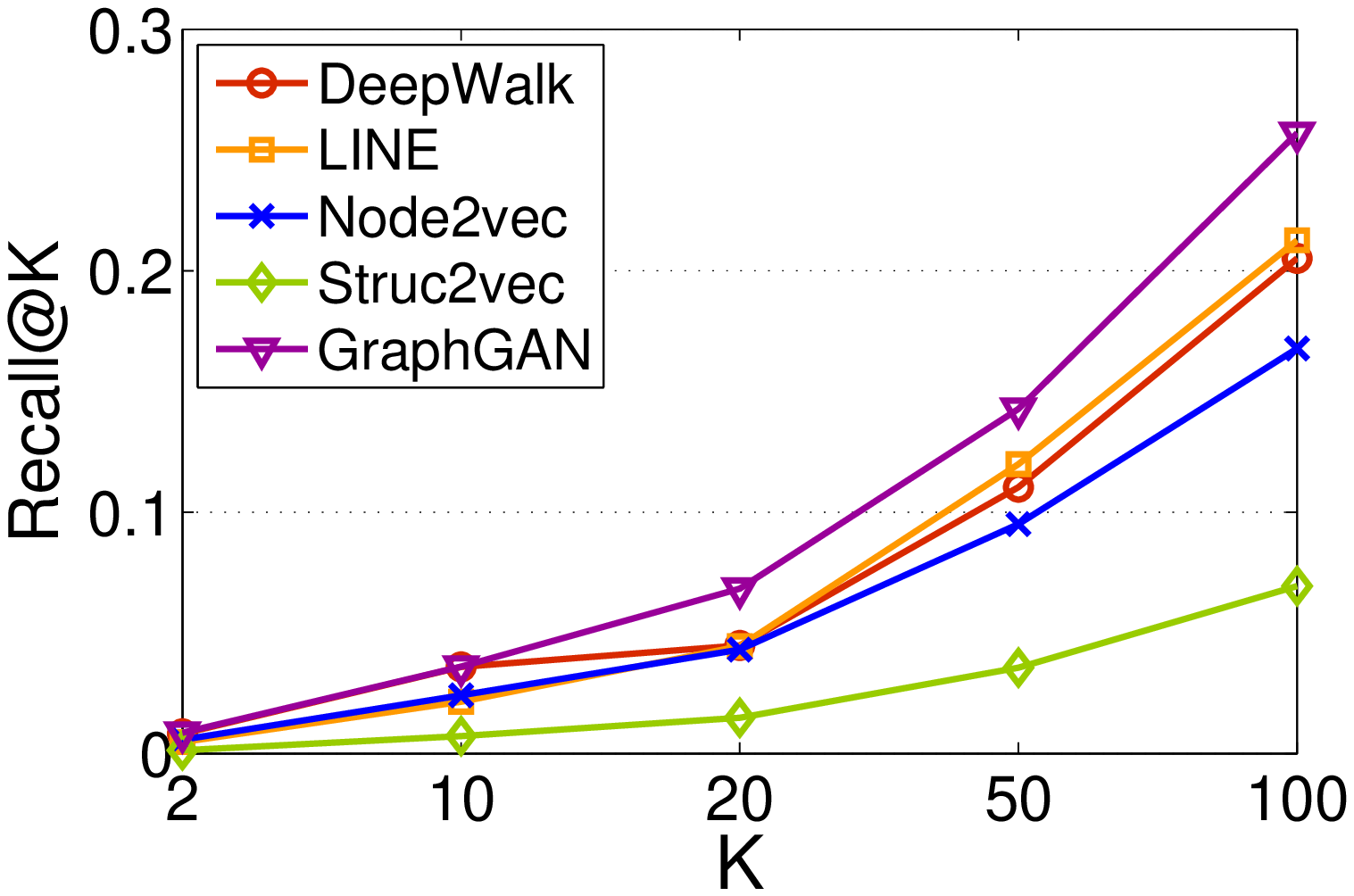}
				\caption{Recall@K}
				\label{fig:recall}
			\end{subfigure}
			\caption{Precision@K and Recall@K on MovieLens-1M in recommendation.}			
			\label{fig:recommendation}
		\end{figure}
		
		We use Movielens-1M as dataset for recommendation.
		For each user, we aim to recommend a set of movies which have not been watched but may be liked by the user.
		We first treat all 4-star and 5-star ratings as edges to obtain a bipartite graph, then randomly hide $10\%$ of edges in the original graph as the test set and construct a BFS-tree for each user.
		Note that different from the above two experiment scenarios where the connectivity distribution is defined on all other vertices for certain vertex, in recommendation the probability of connectivity for one user is only distributed over a fraction of vertices, i.e., all movies, in the graph.
		Therefore, we ``shortcut'' all user vertices in the BFS-tree (except the root) by adding direct edges within all movie pairs that are linked by a user vertex.
		After training and obtaining representations of users and movies, for each user, we select $K$ of his unwatched movies with the highest inner product as the recommendation result.
		The results of \textit{Precision@K} and \textit{Recall@K} are shown in Figure \ref{fig:recommendation}, from which we can observe that GraphGAN is consistently above all baselines, and achieves statistically significant improvements on both metrics.
		Take Precision@20 as an example, GraphGAN outperforms DeepWalk, LINE, node2vec, and struc2vec by $38.56\%$, $59.60\%$, $124.95\%$, and $156.85\%$, respectively.
		Therefore, we can draw the conclusion that GraphGAN maintains a more decent performance in ranking-based tasks compared with other graph representation learning methods.

\section{Conclusions}
	In this paper, we propose \textit{GraphGAN} that unifies two schools of graph representation learning methodologies, i.e., generative methods and discriminative methods, via adversarial training in a \textit{minimax} game.
	Under the GraphGAN framework, both the generator and the discriminator could benefit from each other:
	the generator is guided by the signals from the discriminator and improves its generating performance, while the discriminator is pushed by the generator to better distinguish ground truth from generated samples.
	Moreover, we propose \textit{graph softmax} as the implementation of the generator, which solves the inherent limitations of the traditional softmax.
	We conduct experiments on five real-world datasets in three scenarios, and the results demonstrate that GraphGAN significantly outperforms strong baselines in all experiments due to its adversarial framework and proximity-aware graph softmax.

\section{Acknowledgments}
	This work was partially sponsored by the National Basic Research 973 Program of China under Grant 2015CB352403.

\bibliographystyle{aaai}
\bibliography{reference} 

\end{document}